\documentclass[nohyperref]{article}
\usepackage[accepted]{aistats2022}
\pdfoutput=1

\usepackage{xcolor}

\newcommand{\ignore}[1]{}

\usepackage{amsthm}

\newtheorem{theorem}{Theorem}

\newtheorem{lemma}[theorem]{Lemma}

\theoremstyle{definition}
\newtheorem{definition}{Definition}

\usepackage{hyperref}

\usepackage{graphicx}

\usepackage{makecell}

\usepackage{amsmath,amssymb,mathtools}

\newcommand{\abs}[1]{\left\lvert#1\right\rvert}

\renewcommand{\(}{\left(}
\renewcommand{\)}{\right)}
\renewcommand{\[}{\left[}
\renewcommand{\]}{\right]}

\newcommand{\norm}[2][]{\left\Vert#2\right\Vert_{#1}}

\newcommand{\R}{\mathbb{R}}

\usepackage{dsfont}

\newcommand{\Y}{\mathcal{Y}}

\newcommand{\Normal}{\mathcal{N}}

\newcommand{\reals}{\mathbb{R}}

\newcommand{\MMD}{\text{MMD}}

\DeclareMathOperator*{\E}{\mathbb{E}}

\begin{document}
% \twocolumn[
% \icmltitle{On the Optimization Landscape of Maximum Mean Discrepancy}
% \icmlsetsymbol{equal}{*}

% \begin{icmlauthorlist}
% \icmlauthor{Itai Alon}{hebrew}
% \icmlauthor{Amir Globerson}{tel,google}
% \icmlauthor{Ami Wiesel}{hebrew,google}
% \end{icmlauthorlist}

% \icmlaffiliation{hebrew}{Department of Computer Science, Hebrew University of Jerusalem, Israel}
% \icmlaffiliation{tel}{Department of Computer Science, Tel Aviv University, Israel}
% \icmlaffiliation{google}{Google Research}

% \icmlcorrespondingauthor{Itai Alon}{itai.alon1@mail.huji.ac.il}

% \icmlkeywords{Machine Learning, ICML}

% \vskip 0.3in
% ]

\twocolumn[

\aistatstitle{On the Optimization Landscape of Maximum Mean Discrepancy}

\aistatsauthor{ Itai Alon \And Amir Globerson \And  Ami Wiesel }

\aistatsaddress{ Hebrew University of Jerusalem \And  Tel Aviv University \And Hebrew University of Jerusalem } ]

\begin{abstract}
Generative models have been successfully used for generating realistic signals. Because the likelihood function is typically intractable in most of these models, the common practice is to use ``implicit'' models that avoid likelihood calculation. However, it is hard to obtain theoretical guarantees for such models. In particular, it is not understood when they can globally optimize their non-convex objectives. Here we provide such an analysis for the case of Maximum Mean Discrepancy (MMD) learning of generative models. We prove several optimality results, including for a Gaussian distribution with low rank covariance (where likelihood is inapplicable) and a mixture of Gaussians. Our analysis shows that that the MMD optimization landscape is benign in these cases, and therefore gradient based methods can converge only to global minimizers. 
\end{abstract}

\section{Introduction}

Generative models are parameterized distributions that are trained to match the distribution of observed data. 
Although generative modeling approaches have achieved impressive results, little is known about theoretical properties of training such models. In particular, when does training converge to a good approximation of the underlying distribution, and how many samples are required for such an approximation. 

Due to its practical importance, the problem of generative modeling has received considerable attention. Most generative methods model the data via a distribution $P(x;\theta)$ that depends on a parameter $\theta$ to be estimated from observed data. A key approach in this context is maximum likelihood estimation (MLE), which maximizes $\log{P(x;\theta)}$ with respect to the parameter.  However, in more complex models it is typically hard to calculate the probability of a given point, and thus MLE is not directly applicable. Another route that has turned out to be very effective is via ``implicit'' models. These do not use likelihood directly, but rather generate a sample from the model distribution (since sampling is typically easy) and use some divergence measure to compare the generated and real sample. 

There are many instances of implicit models, including Generative Adversarial Networks (GAN, \cite{goodfellow2014generative}), Wassertsein GAN \cite{gulrajani2017improved}, and Maximum Mean Discrepancy (MMD) based methods which we focus on here. All of these employ a function that takes as input two samples and returns a scalar that 
reflects the ``statistical distance'' between these. For example in GANs this is calculated as the minimum loss of a classifier trained to discriminate between the samples. 

MMD is essentially a measure of similarity between two distributions
\cite{gretton2012kernel,sutherland2017generative, wang2018improving}. It measures similarity by considering a set of functions $f(x)$ (specified by a kernel), and checking how different their expected values are for the two distributions. MMD has several nice properties, including the fact that it is zero if and only if the two distributions are equal, that it does not require min-max optimization and that it can be calculated from samples and does not require explicit likelihood calculations. MMD is also closely related to WGAN \cite{gulrajani2017improved}, where the functions $f(x)$ are optimized over.

Despite their widespread use, little is known about the theoretical properties of MMD based learning of generative models, and in particular its optimization properties. In this paper, we provide a study of MMD learning for several important instances of generative models, showing that the optimization landscape in these cases is in fact benign. 

We analyze the MMD optimization landscape in the following three settings. We begin with Gaussian distributions with an unknown mean, which already introduce non-convexity into the MMD optimization problem. We then consider Gaussians with low-rank covariance matrices, which can be used to capture low dimensional structures in data, and where an MLE does not exist.
We then continue to the more challenging case of mixtures of Gaussian distributions, which capture the multimodal structure that is key to complex generative models. From a theoretical perspectives mixture models are interesting, since they are generally intractable for maximum-likelihood, but are efficiently learnable under certain conditions \cite{daskalakis2017ten} which we also focus on here. 
Gaussian mixtures are already challenging for GAN models \cite{farnia2020gat}, further motivating a theoretical understanding of when they can be learned effectively. 

\textbf{Our contribution:}
We analyze the landscape of MMD optimization for the setting of population loss of the above three settings. We show that in these cases, the landscape is benign, implying that gradient descent (GD) can converge only to global minimizers. We support our theoretical findings with empirical evaluation.

\subsection{Related Work}
MMDs are an instance of Integral Probability Metrics (IPM) \cite{muller1997integral} which measure distance between distributions via the difference of expected values of certain functions. MMD was originally used in the field of machine learning for two sample tests, and much work was devoted to the power of these tests \cite{gretton2012kernel}.

With the growing interest in neural based generators, it became clear that MMDs can also be used for learning such models. This is due to the fact that MMDs can work directly on samples, and do not require likelihood computations. Many different modeling schemes involving MMD have been introduced, including Moment Matching Networks \cite{li2015generative,dziugaite2015training} and MMD-GANs \cite{li2017mmd}, as well as variational autoencoders \cite{rustamov2019closed}. Some MMD based approaches also allow some learning of the kernel parameters to provide better flexibility with respect to the generating distribution \cite{li2017mmd, li2015generative, binkowski2018demystifying, dziugaite2015training, sutherland2017generative, wang2018improving}.

MMDs have been well studied from a statistical perspective. Namely, the number of samples needed for an MMD generative modeling method in order to approximate the true underlying parameters. In this context, there are two relevant sample sizes: those sampled from the true distribution and those sampled from the model. Such  analyzes for dependence between the number of samples used and the ability of the MMD metric to recover the true distribution see \cite{tolstikhin2016minimax, ji2018minimax}.

 Optimization of MMD measures has been recently studied in several works \cite{arbel2019maximum, mroueh2021convergence} where gradient flow formalisms for MMD were studied towards the goal of improved regularization and stability of the optimization. However, they do not study optimization of specific parameteric models as we do here, and do not provide related landscape results.

 Since we consider learning Gaussian mixtures, our work is also related to the long line of work on parameter estimation in mixture models. It is known that learning a mixture of $k$ Gaussians can be done in time polynomial in the dimension and exponential in $k$ \cite{moitra2010settling}. Thus, the $k=2$ case which we study here is tractable (using variants of moment matching techniques), which motivates the question of whether it can be learned using MMD GANs. Furthermore, it was also shown that EM can efficiently learn the parameters in this case \cite{daskalakis2017ten}. The landscape of the likelihood function in this case was also explored in \cite{jin2016local}, where it was shown to have bad local minima for mixtures of three or more Gaussians.

Our main focus in this work is the optimization landscape of MMD. The question of optimization guarantees for ``implicit'' generative models is important due to the many practical applications of optimizing these model. However, not many results have been obtained. In \cite{feizi2017understanding} it was shown that a Gaussian with unknown full-rank covariance is globally optimizable using a quadratic version of a WGAN (see Section \ref{section:gaussian-cov}). In \cite{pmlr-v80-li18d} it was shown that GAN optimization can fail to find the right model because of instabilities related to discriminator optimization. 

Another relevant line of work is on ascent-descent algorithms which underlie GAN optimization \cite{daskalakis2018limit}. But these works typically  show convergence to local optima. Finally, \cite{lei2020sgd} study WGAN optimization for a one hidden layer generative model. Their setting is related to our unknown covariance result (because if one ignores the non-linearity in their case, their model is a Gaussian with unknown covariance). However, our main challenge is the low-rank of the covariance, whereas in their case it is full rank (see Theorem 3 In \cite{lei2020sgd} which requires $k=d$). 

\section{Problem Setup}
Let $P_{\theta}(x)$ denote a distribution on a random variable $X$, which is defined via a parameter vector $\theta$. For example $P_{\theta}(x)$ could be a Gaussian distribution with mean $\theta$. We consider the parameter estimation problem where data is generated i.i.d from some $P_{\theta^*}$ and our goal is to recover $\theta^*$ from these observation. We let $Y$ denote the data generated from $P_{\theta^*}$.

We assume access to a generator that allows us to draw a synthetic dataset of i.i.d samples $X\sim P_\theta$ for any given $\theta$. 
Thus, we would like to find a $\theta$ such that $X\sim P_\theta$ is as close as possible to $Y$.

As mentioned above, MMD is often used to measure the distance between the generated sample $X$ and the ``true'' sample $Y$, as a function of $\theta$. Our goal here is to understand the optimization landscape of this problem. To simplify analysis, we focus on the population case, where the size of both the $X$ and $Y$ samples goes to infinity. In this case, the MMD distance is given by the following:\footnote{More precisely, MMD is formally defined as the squared root of the formula in (\ref{mmd_def}).}
\begin{multline}\label{mmd_def}
    \MMD(\theta,\theta^*) = \E_{X,X' \overset{i.i.d.}\sim P_\theta}[k(X,X')] \\+ \E_{Y,Y' \overset{i.i.d.}\sim P_{\theta^*}}[k(Y,Y')] - 2 \E_{\substack{X \sim P_\theta\\Y \sim P_{\theta^*}\\i.i.d.}}[k(X,Y)]
\end{multline}
where $k$ is a characteristic kernel function (over reproducing kernel Hilbert space). For simplicity, we focus on the Gaussian RBF kernel: 
\begin{equation}
    k(x,y) = e^{-\frac1{2\sigma^2} \norm{x-y}^2}
\end{equation}
where $x,y \in \R^d$ and $\sigma^2$ is a width hyperparameter. 
The MMD optimization problem is: $\min_{\theta} \MMD(\theta,\theta^*)$. It is the landscape of this problem that we would like to understand. Specifically, we would like to understand when optimization methods are expected to converge to $\theta=\theta^*$ and not get stuck at local optima.

The statistical properties of MMD are well understood. Due to the characteristic kernel, MMD is a distance function, i.e., $\theta^*=\theta$ if and only if $\MMD = 0$. To enjoy these statistical benefits in large scale, MMD must be numerically minimized using first order descent methods. Unfortunately, the loss is usually non-convex, and it is not clear whether its optimization is tractable. 

It is well known that GD with a random initialization and sufficiently
small constant step size almost surely converges to a local minimizer under certain conditions \cite{lee2016gradient, jin2017escape, sun2015nonconvex}. Such results require that the saddle points of the loss are ``well behaved'', as captured by the ``strict saddle property'' defined below. 

%Formally, the ``strict saddle property'' defined below, in addition to an upper bound on the Hessian eigenvalues, implies that the optimization will not get stuck at stationary points other than local minimizers. As we show here, for MMD there are multiple cases of interest where the above conditions are satisfied (i.e., no local minima, no strict saddle points, and bounded Hessian), which implies that in these cases GD can only converge to the true underlying parameters.

%which in case of MMD we will show are indeed equivalent to $\theta^*$.
% Formally, convergence requires the ``strict saddle property'' defined below. Thus, to ensure convergence to $\theta^*$ we need to prove that MMD satsifies the property and that all of its local minimizers are indeed equivalent to $\theta^*$. 

\begin{definition}[\cite{lee2016gradient}]
A function $f(\theta)$ satisfies the ``strict saddle property'' if
each critical point $\theta$ of $f$ is either a local minimizer, or a ``strict saddle'', i.e, its Hessian $\nabla^2f(\theta)$ has at least one
strictly negative eigenvalue.
\end{definition}
%\amirg{

In \cite{lee2016gradient} it is shown that if $f$ has the ``strict saddle property'', has a Lipschitz gradient and the iterations have a limit, then GD will converge to a local minimizer. This result is very useful, because it means that if one also proves that the loss function has no local minimizers that are not global, then GD can only converge to a global minimizer.

%it is shown that that if $f$ has the ``strict saddle property'', and $f$ has a Lipschitz gradient, then gradient descent will converge to a local minimizer or infinity almost surely. This result is very useful, because it means that if one proves a loss function has the ``strict saddle property'' and it has no local minimizers that are not global, then if gradient descent converges it is only to a global minimizer.

Here we apply this proof strategy to the case of MMD. In particular, we will show that for several MMD settings of interest, the loss has no bad local minima and satisfies the strict saddle property. The next subsections provide this analysis for three fundamental estimation models.
%}
\section{Theoretical Results}
We next consider three parameter estimation problems, and show that MMD has a benign landscape for those. Namely, the MMD objective has no local minima and all its saddle-points are strict. This implies that gradient based methods can find a global optimum \cite{lee2016gradient}. We focus on the cases of Gaussian with an unknown mean (Sec.~\ref{section:gaussian-mean}), Gaussian with unknown low-rank covariance (Sec.~\ref{section:gaussian-cov}) and a mixture of Gaussians with unknown means (Sec.~\ref{section:gmm-mean}). 
 
\subsection{Gaussian with Unknown Mean}
\label{section:gaussian-mean}
We begin with probably the simplest parameter estimation problem, namely a Gaussian distribution with unknown mean and a known covariance matrix. This problem is trivially solved by MLE. The negative log likelihood is strongly convex and its global minimum, the sample mean, can be efficiently found using GD. The following theorem states that the landscape of MMD is also benign in this case and can be minimized to identify the true mean.

\begin{theorem}
\label{theorem:gaussian-mean-equation}
Let $P_\mu = \Normal(\mu,\Sigma)$ with an unknown $\mu^*$ parameter. The function $\MMD(\mu^*, \mu)$ is given by:
\begin{equation}
     \frac2{\sqrt{\abs{2\frac1{\sigma^2}\Sigma + I}}} \( 1 - e^{-\frac12 \frac1{\sigma^2} (\mu-\mu^*)^T (2\frac1{\sigma^2}\Sigma + I)^{-1} (\mu-\mu^*)} \) ~.
\end{equation}
It is a quasi-convex function of $\mu$, and has a single stationary point at $\mu = \mu^*$, which is the global minimum.
\end{theorem}
\begin{proof}
We rely on the following integral (see Supplementary Material, Lemma \ref{lemma:chi_square}):
\begin{multline}\label{integral}
    \int_{\R^d} e^{-\frac12\frac1{\sigma^2} \norm{z}^2} \cdot \Normal(z;\mu,\Sigma) dz
        \\= \abs{\frac1{\sigma^2}\Sigma + I}^{-\frac12} e^{-\frac12 \frac1{\sigma^2} \mu^T (\frac1{\sigma^2}\Sigma + I)^{-1} \mu}
\end{multline}
We compute each expectation in (\ref{mmd_def}) separately. If $X, X' \sim \Normal(\mu, \Sigma), Y, Y' \sim \Normal(\mu^*, \Sigma)$ are independent, then $X - Y \sim \Normal(\mu-\mu^*, 2\Sigma)$, and $X-X'$, $Y-Y' \sim \Normal(0, 2\Sigma)$. Using (\ref{integral}) we have that $\E[k(X,Y)]$ is equal to:
    \begin{equation}
         \frac1{\abs{2\frac1{\sigma^2}\Sigma + I}^\frac12} e^{-\frac12 \frac1{\sigma^2} (\mu - \mu^*)^T (2\frac1{\sigma^2}\Sigma + I)^{-1} (\mu - \mu^*)}
    \end{equation}
    and: $\E[k(X,X')] = \E[k(Y,Y')] = \frac1{\abs{2 \frac1{\sigma^2} \Sigma + I}^\frac12}$.
\end{proof}

The theorem implies that standard descent methods on MMD can easily find the unknown Gaussian mean just like MLE, {because of quasi-convexity and lack of saddle points \cite{hazan2015beyond}.} Note however that MLE is strongly convex whereas MMD is quasi-convex as illustrated in Figure \ref{fig:gaussian_mmd}. The gradient vanishes if we initialize the algorithm too far from the true $\mu^*$ and the optimization is much harder. On the other hand, MMD has a width hyperparameter that can be used to widen the convex region of the landscape. For example, if we know a priori that $\mu^*$ is within a ball of radius $r$ around the initial point, then choosing a width $\sigma^2\geq 2r$ ensures a strongly convex objective where GD converges linearly.

\begin{figure}[t]
\centerline{\includegraphics[width=.5\textwidth]{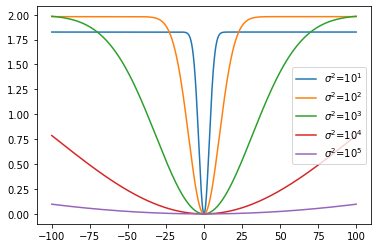}}
\caption{MMD loss of Gaussian over $\R$ with unknown mean. Different widths $\sigma^2$ and $\mu^*=0$.}
\label{fig:gaussian_mmd}
\end{figure}

\subsection{Gaussian with Unknown Covariance}
\label{section:gaussian-cov}
Our second model is a zero mean multivariate normal with an unknown singular covariance matrix. This is a more challenging setting which is more representative of modern deep learning problems where an MLE does not exist.  Even in the near-singular case, GD on the MLE loss is very slow, because the problem is ill conditioned.
The standard adversarial GAN also fails in this task \cite{cho2019wasserstein}.
 For WGAN, it has been shown that its global optimum is correct 
 \cite[e.g., see][]{feizi2017understanding,cho2019wasserstein} for two results in this flavor). However, it was not shown that this optimum can be found using GD, and in fact \cite{feizi2017understanding} (Theorem 5) points to a failure of optimization in this case.

The following theorem states that the landscape of MMD in this model is benign. The constant $\epsilon$ determines the level of covariance non-singularity, and results also hold for $\epsilon=0$. We consider the case of rank one, since it captures the most ``extreme'' low-rank scenario.
\begin{theorem}
\label{theorem:cov-gaussian-landscape}
Consider the parametric model:
\begin{equation}
    P_{a} = \Normal(0, a{a}^T+\varepsilon^2 I)
    \label{eq:cov_model}
\end{equation} with an unknown $a\in\reals^d$ parameter. Assume training data is generated from $P_{a^*}$. The MMD objective is
    \begin{multline}
    \label{eq:mmd-cov}
    \MMD({a^*}, a) = \\\frac1{\abs{\frac1{\sigma^2} (2aa^T + 2\varepsilon^2I) + I}^\frac12} + \frac1{\abs{\frac1{\sigma^2} (2a^*{a^*}^T + 2\varepsilon^2I) + I}^\frac12} \\- \frac2{\abs{\frac1{\sigma^2} (aa^T + a^*{a^*}^T + 2\varepsilon^2I) + I}^\frac12}
\end{multline}
It has global minima in $\pm a^*$, a maximum in $a=0$, and the rest of the stationary points are (if the radius exists)
\begin{equation}
    \left\{a\ :
    \begin{array}{l}
         a^T a^* = 0  \\
         \norm{a}^2 = \frac{(2\varepsilon^2 + \sigma^2)\((\norm{a^*}^2 + 2\varepsilon^2 + \sigma^2)^\frac13 - (2\varepsilon^2 + \sigma^2)^\frac13 \)}{2 (2\varepsilon^2 + \sigma^2)^\frac13 - (\norm{a^*}^2 + 2\varepsilon^2 + \sigma^2)^\frac13}
    \end{array}\right\}
\end{equation}
and are all strict saddles.
\end{theorem}

\begin{proof}
   The derivation of (\ref{eq:mmd-cov}) is straight forward using (\ref{integral}) for each expectation.

    To simplify the landscape analysis, we rotate the space and assume w.l.o.g that $a^* = e_1$ is the first unit vector with $\norm{a^*}^2 = 1$. Denoting $c := 2\varepsilon^2 + \sigma^2$,
    Lemma \ref{lemma:gauusian-cov-gradient} in the Supplementary Material provides the gradient:
    \begin{multline}
        0 = \nabla \MMD(P_{a^*}, P_a) \propto - \abs{2aa^T + c I}^{-\frac32} a \\+  \abs{aa^T + {a^*}{a^*}^T + c I}^{-\frac32} (1 + c) ({a^*}{a^*}^T + c I)^{-1} a
    \end{multline}
    Rearranging and isolating $a$ yields
    \begin{equation}
    \label{eq:cov_eigans}
        ({a^*}{a^*}^T + c I)^{-1} a = \frac{\abs{aa^T + {a^*}{a^*}^T + c I}^\frac32}{\abs{2aa^T + c I}^\frac32} (1 + c)^{-1} a
    \end{equation}
    The critical points are all the solutions to this linear system. These are the eigenvectors of $({a^*}{a^*}^T + c I)^{-1}$. Using Woodbury's identity:
    \begin{equation}
        \(a^*{a^*}^T + c I\)^{-1} = 
        c^{-1} I - c^{-1} (1 + c)^{-1} {a^*}{a^*}^T \nonumber
    \end{equation}
    It easy to see that one solution is $a = t \cdot a^*$. Together with equation (\ref{eq:cov_eigans}) one can conclude that $t = \pm 1$. Moreover, $a=\pm a^*$ yields a zero loss, and because MMD is non-negative we also conclude that these are global minima.
    The origin $a=0$ is a trivial solution, and by placing zero in the Hessian (Lemma \ref{lemma:gausian-cov-hessian}) it is clearly a maxima.
    All the other $d-1$ eigenvectors are in the orthogonal hyperplane of $a^*$. Plugging $a \perp a^*$ into (\ref{eq:cov_eigans}) we get after some algebra:
    \begin{equation}
            c^{-1}a = \frac{\((\norm{a}^2 + c) (1 + c)\)^\frac32}{\((2\norm{a}^2 + c)c\)^\frac32} (1 + c)^{-1}a
    \end{equation}
   For $a\neq 0$ this leads to
    \begin{equation}
        \label{eq:cov_a_norm}
        \norm{a}^2 = \frac{c\((1 + c)^\frac13 - c^\frac13 \)}{2 c^\frac13 - (1 + c)^\frac13}
    \end{equation}
    In the Supplementary Material we prove that (\ref{eq:cov_a_norm}) ensures  ${a^*}^T H(a) a^* < 0$.  Thus, $a^*$ is a descent direction and the orthogonal stationary points are not minima. 
\end{proof}

The theorem proves that the landscape of MMD with unknown covariance is benign. Figure \ref{fig:gaussian_cov_bandwidths} illustrates this in a one dimensional setting. As before, it can be seen that there are no bad local minima, yet the gradient vanishes as we move further from $a^*$.\footnote{This happens for the case $\sigma^2=0.1$ for larger $a$ values.} 

{In order to use the results from \cite{lee2016gradient}, we also need the loss to have a Lipschitz gradient. This indeed follows from Lemma \ref{lemma:cov_bounded} in the Supplementary, where we show that  the eigenvalues of the MMD Hessian are upper bounded by a constant (namely, we upper bound $\norm[2]{\nabla^2 \MMD(a)}$), and therefore the gradient is Lipschitz. Thus we conclude that in this MMD case, GD will not converge to saddle points or bad local minima.}

\begin{figure}[t]
\centerline{
\includegraphics[width=.5\textwidth]{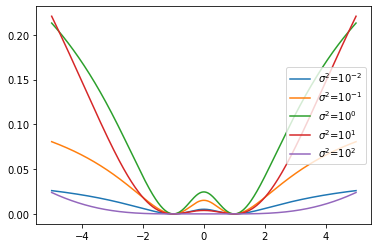}
}
\caption{MMD loss of Gaussian over $\R$ with unknown covariance, different widths $\sigma^2$ and $a^*=1$. }
\label{fig:gaussian_cov_bandwidths}
\end{figure}

\subsection{Mixture of Two Gaussians with Unknown Mean}
\label{section:gmm-mean}
We next study Gaussian mixture models with symmetric unknown means. MLE optimization in general GMMs is known to be non-convex and NP-hard \cite{aloise2009np}. Three component GMMs can be hard even for MLE with infinitely many samples \cite{jin2016local}. The symmetric case is much easier and can be optimized efficiently using an Expectation Maximization algorithm \cite{daskalakis2017ten}. In terms of likelihood-free estimators, GANs are known to fail in GMMs and special purpose discriminators are required \cite{farnia2020gat}. 

The following theorem shows that for symmetric Gaussian mixtures, the MMD landscape is in fact benign.
%The following surprising theorem shows that the landscape of MMD is benign in this case.

\begin{theorem} \label{theorem:gmm-mean-landscape}
Consider the following family of distributions where $x\in\reals^d$ 
\begin{equation}
    p(x) = 0.5 \cdot \Normal(x;\mu,\Sigma) + 0.5 \cdot \Normal(x;-\mu,\Sigma)
\end{equation}
Let $\mu^* \ne 0$ be an unknown parameter and $\Sigma\in\reals^{d\times d}$ a known covariance matrix. Then,
\begin{multline}
    \MMD(P_{\mu^*}, P_\mu) = \\\frac12 \frac{\sigma^d}{\sqrt{\abs{2 \Sigma + \sigma^2 I}}} \Big[ 
        e^{-\frac12 (2\mu)^T (2\Sigma + \sigma^2 I)^{-1} (2\mu)} + 1
        \\+ e^{-\frac12 (2\mu^*)^T (2\Sigma + \sigma^2 I)^{-1} (2\mu^*)} + 1
        \\-  2 \cdot e^{-\frac12 (\mu - \mu^*)^T (2\Sigma + \sigma^2 I)^{-1} (\mu - \mu^*)}
        \\- 2 \cdot e^{-\frac12 (\mu + \mu^*)^T (2\Sigma + \sigma^2 I)^{-1} (\mu + \mu^*)}
    \Big]
\end{multline}

The function $\MMD(P_{\mu^*}, P_\mu)$ has global minima in $\pm\mu^*$, a global maximum in $0$, and the other stationary points are
\begin{equation}
 \left\{ \mu : \begin{array}{l}
      \mu ^T(2\Sigma + \sigma^2 I)^{-1} \mu^* = 0  \\
      \mu^T (2\Sigma + \sigma^2 I)^{-1} \mu = {\mu^*}^T (2\Sigma + \sigma^2 I)^{-1} \mu^* 
 \end{array}\right\}   
\end{equation}
and are all strict saddles.
\end{theorem}

\begin{proof}
    By defining $\mu := \Sigma^{-\frac12} \mu$, we can consider w.l.o.g the case of $\Sigma = I$. Define $c = 2 + \sigma^2$. To analyze the landscape, we differentiate MMD and compare to zero:
    \begin{multline}
        0  = \nabla \MMD(P_{\mu^*}, P_\mu) \propto 
            e^{-\frac12 c^{-1} \norm{\mu - \mu^*}^2} (\mu - \mu^*)
            \\+ e^{-\frac12 c^{-1} \norm{\mu + \mu^*}^2} (\mu + \mu^*)
            - 2 e^{-\frac12 c^{-1} \norm{2\mu}^2} \mu
    \end{multline}
    which holds when either (A) $\mu = t \cdot \mu^*$ for some $t \in \R$; or (B) when the coefficients vanish.
    \\\\
    In case (A), denoting $\alpha = e^{-\frac12 c^{-1} \norm{\mu^*}^2}$, the derivative as a function of $t$ is given by
    \begin{multline}
        \nabla \MMD(P_{\mu^*}, P_{t \cdot \mu^*}) \propto 
            \\(t-1) \cdot \alpha^{(t-1)^2} + (t+1) \cdot \alpha^{(t+1)^2} - 2t \cdot \alpha^{4t^2}
    \end{multline}
    which is zero in $t \in \{-1,0,1\}$, where by exploration of ascent and descent areas $\pm 1$ are minima, and $0$ is a local maximum:
    $g(t) < 0$ for $t \in (-\infty,-1)\cup(0,1)$, $g(t) > 0$ for $t \in (-1,0)\cup(1,\infty)$ and $g(t)=0$ for $t \in \{-1,0,1\}$.
    \\\\
    In case (B), the coefficients vanish when
    \begin{align}
        &\norm{\mu - \mu^*}^2 = \norm{\mu + \mu^*}^2\nonumber
\end{align}
and 
\begin{align}
        &\norm{\mu - \mu^*}^2 + \norm{\mu + \mu^*}^2 = 4 \norm{\mu^*}^2
    \end{align}
    Thus,
    \begin{equation}\label{orthcond}
        \mu^T \mu^* = 0,\qquad \norm{\mu^*}^2 = \norm{\mu}^2
    \end{equation}
    To characterize $\mu$ in this case, we compute the Hessian and plugging (\ref{orthcond}) yields:
    \begin{multline}
        H = 2 \( e^{-c^{-1} \norm{\mu^*}^2} - e^{-2 c^{-1} \norm{\mu^*}^2} \) c^{-1} I 
        \\- 2 e^{- c^{-1} \norm{\mu^*}^2} c^{-2} [\mu \mu^T + \mu^* {\mu^*}^T] \\
        + 8 e^{-2 c^{-1} \norm{\mu^*}^2} c^{-2} \mu \mu^T
        \label{eq:hessian_mixture_main}
    \end{multline}
    Then, it is easy to see that $\mu^*$ is an eigenvector and its corresponding eigenvalue is negative. The full details are available in the Supplementary Material.
\end{proof}

The benign landscape of MMD in a symmetric GMM is illustrated in Figure \ref{fig:gmm_mmd}. As before, MMD suffers from vanishing gradient far from the origin. With poor width choices, there are also troubling deflection points which have small yet non-zero gradients. By tuning the width parameter accordingly, we can upper bound the Hessian (Supplementary Material, Lemma \ref{lemma:mean_bounded}). {As in Section \ref{section:gaussian-cov}, it follows we can use the results in \cite{lee2016gradient}, and conclude that GD will not converge to saddle points or bad local minima.}

\begin{figure}[t]
\centerline{\includegraphics[width=.5\textwidth]{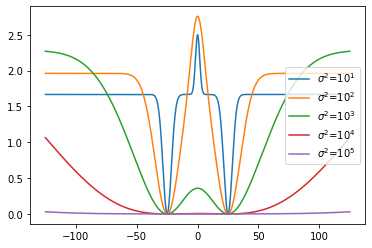}}
\caption{MMD loss of GMM over $\R$ with different widths $\sigma^2$ and $\mu^* = 25$. Note that in $\sigma^2 = 10^1$ around $\mu=\pm 8.5$ there are inflection points. The derivatives at these points are small but non-zero.
}
\label{fig:gmm_mmd}
\end{figure}

\section{Experiments}
 %\hl{if convergence, as we will show empirically}
The results above show that MMD is expected to converge globally for the three distributions we consider. As mentioned earlier, MLE optimization also has such guarantees (except for the singular covariance case). However, no such results are available for WGAN. Here we report empirical evaluations that are in line with these observations. In all cases we denote by $y_1,\ldots,y_m$ the sample from the true distribution (with parameters $\theta^*$).

We consider the following approaches to learning the generative models we analyzed: \begin{itemize}
    \item MMD: This is the standard implementation of the MMD approach. Here, at each iteration a ``fake'' sample $x_1,\ldots,x_n$ is generated from the current parameter estimate $\theta$ and the following finite-sample version of the MMD loss (see \eqref{mmd_def}) is optimized:\footnote{The bandwidth is chosen via holdout from range of powers of ten.}
    \begin{multline}\label{finiteMMD}
        \MMD_{m,n} = \frac1{n \cdot (n-1)} \sum_{i \ne j} k(x_i,x_j) \\+ \frac1{m \cdot (m-1)} \sum_{i \ne j} k(y_i, y_j) - \frac2{n \cdot m} \sum_{i,j} k(x_i, y_j)
    \end{multline}
    
    \item One-Sided MMD (OSMMD): Because we consider Gaussian distributions it is possible to take the limit of $m\to\infty$ and compute the integrals in closed form (this is essentially what we do in the analysis of the landscape). Thus, OSMDD is defined as in (\ref{finiteMMD}) but with $n=\infty$ via an analytically computed expectation.
    \item Maximum Likelihood Estimation (MLE): Here we maximize the standard log-likelihood function.
    \item Wasserstein GAN (WGAN): This is the standard WGAN implementation where generator parameters and discriminator parameters are trained via gradient descent-ascent. The discriminator is structured as linear in the Gaussian mean experiment and as a two layers with ReLU activation function in the other experiments. The hidden dimension is chosen to be the lowest power of two which is greater than the problem dimension. 
\end{itemize}

Implementation details: All experiments run on Pytorch using Adam with learning rate $10^{-1}$. All the algorithms use a full batch of real samples. MMD and WGAN use $n=m$ fake samples which are independently generated in each epoch.

\subsection{Gaussian with Unknown Mean}
We first consider the estimation of an unknown mean as in Sections \ref{section:gaussian-mean} and \ref{section:gmm-mean}. In each simulation, we compute the distance  $\norm{\mu - \mu^*} / d$ and report the number of times this distance is less than $0.02$. 

\begin{figure}[t]
\centerline{\includegraphics[width=.4\textwidth]{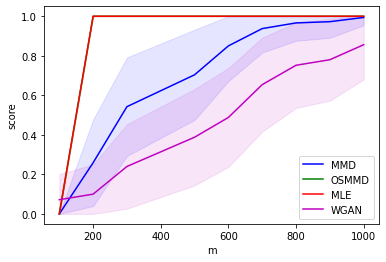}}
\caption{Success rate in a Gaussian model with unknown mean as a function of $m$. OSMMD and MLE coincide.} 
\label{fig:gaussian_mean_resualts}
\end{figure}

\begin{figure}[t]
\centerline{\includegraphics[width=.4\textwidth]{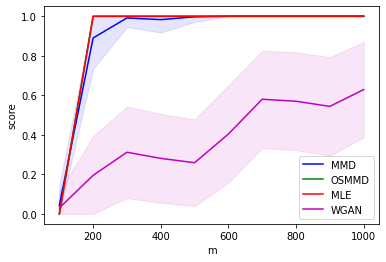}}
\caption{Success rate in a 2-GMM model with unknown mean as a function of $m$. OSMMD and MLE coincide.} 
\label{fig:gmm_mean_resualts}
\end{figure}

Results are shown in Figure \ref{fig:gaussian_mean_resualts}. As expected, MLE easily succeeds to recover the unknown parameter.
Among the implicit likelihood-free generators, MMD significantly outperforms WGAN, apparently because of optimization failure of the latter. OSMMD coincides with MLE and suggests that the gap to optimality can be eliminated by increasing $n$. Similar results and conclusions are provided in Figure \ref{fig:gmm_mean_resualts} for a symmetric GMM with unknown mean. 

\subsection{Gaussian with Unknown Covariance}
We next consider the estimation of an unknown near-singular covariance $a^* {a^*}^T + \varepsilon^2 I$. This is more representative of deep neural nets which model low-dimensional data manifolds. Furthermore, for the case $\epsilon=0$ the MLE estimator does not exist, and thus it is of particular interest to understand other estimators.

In the figures, we measure the accuracy of a parameter estimation method as the fraction of cases where
 $\|aa^T - {a^*}{a^*}^T\| / d \leq \gamma$ where $\gamma$ is a threshold and $\gamma=0.5 \cdot 10^{-1}$ for the single Gaussian case and $\gamma = 1 \cdot 10^{-1}$ for the GMM.

We begin by verifying the intuition that MLE’s success depends on the singularity parameter $\varepsilon$. In Figure \ref{fig:cov_likelihood_vs_mmd} we compare the performance of MMD and MLE as parameter of $\varepsilon$. Theoretically, MLE does not exist for $\varepsilon=0$ and the graph clearly shows that it fails to recover the unknown covariance also in the near-singular case. Both algorithms use the same  number of iterations and learning rate.  

Figure \ref{fig:gaussian_cov_resualts} shows the performance of the baselines for the covariance estimation problem with zero mean. MLE fails in all the experiments due to the small $\varepsilon=1\cdot10^{-5}$. OSMMD is significantly better than all the methods, and the more realistic MMD is second best. WGAN performs worse than the MMD baselines and also requires significant hyperparameter tuning.

We next consider the case of unknown covariance, but with a mixture of two Gaussians. Namely the model:
\begin{equation}\label{eq:gmm-cov}
    p(x) = 0.5 \sum_{z\in\{-1,+1\}}\Normal(x;z \mu,a^* {a^*}^T + \varepsilon^2 I) \nonumber
\end{equation}
where $\mu$ and $\varepsilon$ are known.
Figure \ref{fig:gmm_cov_resualts} shows the results for this case, and the qualitative behavior is similar to that of a single Gaussian. We note that for WGAN we explored a wide range of parameter settings (step sizes, initialization, optimization algorithms and architectures), and yet its performance could not be improved beyond that reported.

\begin{figure}[t]
\centerline{\includegraphics[width=.4\textwidth]{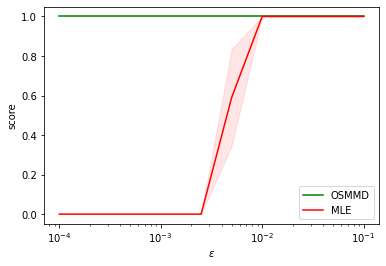}}
\caption{MLE vs MMD as a function of $\varepsilon$ for covariance estimation. See Equation \eqref{eq:cov_model}.}
\label{fig:cov_likelihood_vs_mmd}
\end{figure}

\begin{figure}[t]
\centerline{\includegraphics[width=.4\textwidth]{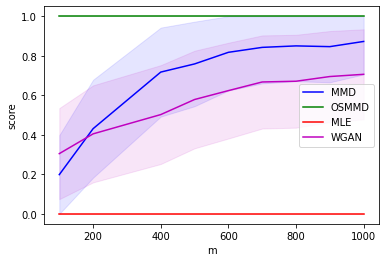}}
\caption{Success rate in a Gaussian model with unknown covariance as a function of $m$.}
\label{fig:gaussian_cov_resualts}
\end{figure}

\begin{figure}[t]
\centerline{\includegraphics[width=.5\textwidth]{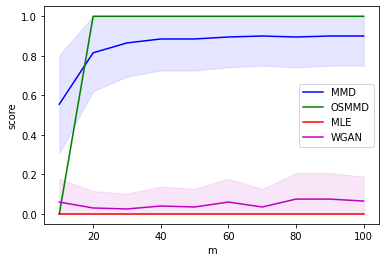}}
\caption{Success rate in a 2-GMM model with unknown covariance as a function of $m$.}
\label{fig:gmm_cov_resualts}
\end{figure}

\subsection{Linear Unmixing \label{sec:unmixing}}
We conclude the experiments with a more challenging linear unmixing setting involving continuous hidden variables \cite{bioucas2012hyperspectral,nascimento2005vertex}. Accurate evaluation of the likelihood function in such models involves a high dimensional integration, and MLE is usually computationally expensive. The underlying model is:
\begin{align}
    x_i &= Ab_i+w_i\nonumber  \ ,  \  b_i= {\mathcal{D}}{\rm{irichlet}}(1_r)\nonumber\\
    w_i&= {\mathcal{N}}(0,\sigma_w^2I)\quad i=1,\cdots n
\end{align}
where $x_i$ are the length $d$ observed vectors, $A$ is a $d \times r$ unknown deterministic matrix to be recovered, $b_i$ are hidden random vectors within the $r$ dimensional simplex and $w_i$ are length $d$ noise vectors.
Using matrix notation, this is a variant of non-negative matrix factorization (NMF): $X=AB+W$ where $X=[x_1,\cdots,x_n]$, $B=[b_1,\cdots,b_n]$ and $W=[w_1,\cdots,w_n]$.
The goal is to estimate $A$ given $X$. 

Linear unmixing is common in hyperspectral imaging where each sample is a pixel and the columns of $A$ represent different materials. Direct computation of MLE is difficult but there are many heuristics and approximations  \cite{bioucas2012hyperspectral}. For our purposes, we only compare the general purpose MMD and WGAN. We experimented with different hyperparameter settings (learning rate, bandwidth, and number of layers in the discriminator) and report the results of the best in each  model. Dimensions were $d=10$, $r=3$ and $n=100$. Both algorithms are initialized using the popular Vertex Component Analysis (VCA) algorithm \cite{nascimento2005vertex}\footnote{https://github.com/Laadr/VCA} and perform $5000$ epochs.

The accuracy of an estimate $A$ with respect to a true $A^*$ is defined using:
\begin{equation}
    d(A, A^*)=\min_{\pi\in\Pi}\sum_i\|a_i-a^*_{\pi_i}\|^2
\end{equation}
where $\Pi$ is the set of all the permutations of $(1,\cdots,r)$ since it is impossible to identify the order of the columns in $A$. Results are provided in Table \ref{unmixing_table} and show a clear advantage for MMD over WGAN.

Additional experiments (not shown) also suggest that MMD is less sensitive to hyperparameter tuning than WGAN, but both require a clever initialization such as VCA and perform poorly otherwise.

\begin{table}[]
    \centering
    \begin{tabular}{|l|l|l|l|}
    \hline
    $\sigma_w^2$  & 0.001 & 0.01 & 0.1 \\
    \hline
WGAN & 0.60 (0.20) & 0.61 (0.19) & 0.61 (0.16) \\
MMD & 0.46 (0.19) & 0.46 (0.17) & 0.50 (0.18) \\
    \hline
    \end{tabular}
    \caption{Accuracy in linear unmixing with different noise variances. See Section \ref{sec:unmixing} for details.}
    \label{unmixing_table}
\end{table}

\section{Conclusions}
Deep learning approaches typically result in non-convex optimization problems.
There are clearly cases where this can lead to bad local optima and therefore failure of the learning method. On the other hand, in practice useful models are often learned. Thus, a key challenge for 
theoretical work is to characterize when such methods are expected to work. 
Clearly this is a technical challenge, due to the non-convexity of the optimization landscape, and indeed not many such results have been proven to date. 

Here we focus on models that we believe to be building blocks of more complex cases. One case is Gaussians with low-rank covariance, which captures the
structure of low dimensional manifolds that are often present in real data. The other is mixtures of Gaussians which capture the common case of multi-modal structure. For these cases, we show that MMD can be globally optimized despite its non-convexity. Interestingly, MMD also offers an advantage over MLE in the covariance learning case.

One challenge with our theoretical results is that in some cases gradients will vanish as the parameter diverges. An interesting approach to ameliorate this is to choose a bandwidth where gradients are sufficiently large. This will result in a method closer to MMD-GAN, and it will be very interesting to explore its theoretical properties for the distributions considered here.

Another avenue for future research is to study the effect of the ``fake'' sample size on both the optimization landscape and the sample complexity. Our empirical evaluation suggests that MMD performs well for finite samples (though not as well as the infinite fake sample; namely OSMMD). It should be possible to analyze this via Morse theory \cite{mei2018landscape} which we leave for future work.

\paragraph{acknowledgements.}
This research was partially supported by ISF grant 2672/21. Also thank to Nerya Granot who supply the VCA code.

\bibliography{main.bib}

\newpage
\appendix
\onecolumn

\providecommand{\upGamma}{\Gamma}
\providecommand{\uppi}{\pi}

\section{PROOFS OF MAIN RESULTS}

\subsection{End of Proof of Theorem \ref{theorem:cov-gaussian-landscape}}
\begin{proof}
    Computing the Hessian (Lemma \ref{lemma:gausian-cov-hessian}) and dropping positive constants:
    \begin{multline}
        H = \nabla^2 \MMD(a)
        \propto (2\norm{a}^2 + c)^{-\frac32} c^{-\frac32} \( 6 \(2\norm{a}^2 + c\)^{-1} a a^T - I\) 
        \\+ (\norm{a}^2 + c)^{-\frac32} (1 + c)^{-\frac12} \Bigg( ({e_1}{e_1}^T + c I)^{-1} - 3 c^{-1} (\norm{a}^2 + c)^{-1} a a^T \Bigg)
    \end{multline}
    Finally we show that (\ref{eq:cov_a_norm}) ensures that $a^*$ is a descent direction:
    \begin{equation}
    \begin{aligned}
        {a^*}^T H a^*
        &\propto - (2\norm{a}^2 + c)^{-\frac32} c^{-\frac32} + (\norm{a}^2 + c)^{-\frac32} (1 + c)^{-\frac32} \\
        &= - \(2\frac{c\((1 + c)^\frac13 - c^\frac13 \)}{2 c^\frac13 - (1 + c)^\frac13} + c\)^{-\frac32} c^{-\frac32} + \(\frac{c\((1 + c)^\frac13 - c^\frac13 \)}{2 c^\frac13 - (1 + c)^\frac13} + c\)^{-\frac32} (1 + c)^{-\frac32} \\
        &\propto - \(2\frac{\((1 + c)^\frac13 - c^\frac13 \)}{2 c^\frac13 - (1 + c)^\frac13} + 1\)^{-\frac32} c^{-\frac32} + \(\frac{\((1 + c)^\frac13 - c^\frac13 \)}{2 c^\frac13 - (1 + c)^\frac13} + 1\)^{-\frac32} (1 + c)^{-\frac32} \\
        &= - \(\frac{(1 + c)^\frac13}{2 c^\frac13 - (1 + c)^\frac13}\)^{-\frac32} c^{-\frac32} + \(\frac{c^\frac13}{2 c^\frac13 - (1 + c)^\frac13}\)^{-\frac32} (1 + c)^{-\frac32} \\
        &= - \(\frac{(1 + c)^\frac13}{2 c^\frac13 - (1 + c)^\frac13}\)^{-\frac32} c^{-\frac32} + \(\frac{c^\frac13}{2 c^\frac13 - (1 + c)^\frac13}\)^{-\frac32} (1 + c)^{-\frac32} \\
        &= \frac{c^{-\frac12} (1 + c)^{-\frac32} - (1 + c)^{-\frac12} c^{-\frac32}}{\(2 c^\frac13 - (1 + c)^\frac13\)^{-\frac32}}
        = \frac{(c (1 + c)^3)^{-\frac12} - ((1 + c) c^3)^{-\frac12}}{\(2 c^\frac13 - (1 + c)^\frac13\)^{-\frac32}} < 0
    \end{aligned}
    \end{equation}
\end{proof}

\vfill

\subsection{End of Proof of Theorem \ref{theorem:gmm-mean-landscape}}
\begin{proof}
     In case (B), $\norm{\mu - \mu^*}^2 = \norm{\mu + \mu^*}^2$, and by plugging in the equation again:
    \begin{equation}
        \norm{\mu - \mu^*}^2 + \norm{\mu + \mu^*}^2 = 4 \norm{\mu^*}^2
    \end{equation}
    which means
    \begin{equation}\label{orthcond2}
        \mu^T \mu^* = 0,\qquad \norm{\mu^*}^2 = \norm{\mu}^2
    \end{equation}
    To characterize $\mu$ in this case, we compute the Hessian:
    \begin{multline}
        H = \nabla^2 \MMD(P_{\mu^*}, P_\mu) =
        e^{-\frac12 c^{-1} \norm{\mu - \mu^*}^2} c^{-1} I - e^{-\frac12 c^{-1} \norm{\mu - \mu^*}^2} c^{-2} (\mu - \mu^*) (\mu - \mu^*)^T
        \\+ e^{-\frac12 c^{-1} \norm{\mu + \mu^*}^2} c^{-1} I - e^{-\frac12 c^{-1} \norm{\mu+ \mu^*}^2} c^{-2} (\mu + \mu^*) (\mu + \mu^*)^T
        - 2 e^{-2 c^{-1} \norm{\mu}^2}  c^{-1} I + 8 e^{-2 c^{-1} \norm{\mu}^2} c^{-2} \mu \mu^T
    \end{multline}
    Using (\ref{orthcond2}) yields:
    \begin{equation}
        H = 2 \( e^{-c^{-1} \norm{\mu^*}^2} - e^{-2 c^{-1} \norm{\mu^*}^2} \) c^{-1} I 
        - 2 e^{- c^{-1} \norm{\mu^*}^2} c^{-2} [\mu \mu^T + \mu^* {\mu^*}^T]
        + 8 e^{-2 c^{-1} \norm{\mu^*}^2} c^{-2} \mu \mu^T
    \end{equation}
    Then, it is easy to see that $\mu^*$ is an eigenvector and its corresponding eigenvalue is negative:
    \begin{align}
        \mu^* H \mu^* &= 2 \( e^{-c^{-1} \norm{\mu^*}^2} - e^{-2 c^{-1} \norm{\mu^*}^2} \)  c^{-1} \norm{\mu^*}^2 
        - 2 e^{- c^{-1} \norm{\mu^*}^2} c^{-2} \norm{\mu^*}^4 \\
        &= 2 e^{-c^{-1} \norm{\mu^*}^2} \( 1 - e^{-c^{-1} \norm{\mu^*}^2}   
        - c^{-1} \norm{\mu^*}^2
        \) c^{-1} \norm{\mu^*}^2 < 0
    \end{align}
    Which holds by Taylor and $\mu^* \ne 0$.
\end{proof}

\newpage

\section{LEMMAS}
\begin{lemma} \label{lemma:chi_square} Let $\mu \in \R^d$ and $\Sigma \in \R^{d \times d}$ a positive definite matrix, then:
\begin{equation}
    \int_{\R^d} e^{-\frac12\frac1{\sigma^2} \norm{z}^2} \cdot \Normal(z;\mu,\Sigma) dz
        = \abs{\frac1{\sigma^2}\Sigma + I}^{-\frac12} e^{-\frac12 \frac1{\sigma^2} \mu^T (\frac1{\sigma^2}\Sigma + I)^{-1} \mu}
\end{equation}
\end{lemma}
\subsection{Proof of Lemma \ref{lemma:chi_square}}
\begin{proof}
    \begin{align*}
        \int_{\R^d} e^{-\frac12\frac1{\sigma^2} \norm{z}^2} \cdot \Normal(z;\mu,\Sigma) dz
        &= \int_{\R^d} \frac1{\sqrt{(2\pi)^d}} \frac1{\abs{\Sigma}^\frac12} e^{-\frac12 \frac1{\sigma^2} \norm{z}^2 } e^{-\frac12 (z - \mu)^T \Sigma^{-1} (z - \mu)} dz
        \\
        &= \frac1{\abs{\Sigma}^\frac12} e^{-\frac12 \frac1{\sigma^2} \mu^T \Sigma^{-1} \mu} \cdot \int_{\R^d} \frac1{\sqrt{(2\pi)^d}} e^{-\frac12 [z^T (\Sigma^{-1} + \frac1{\sigma^2}I) z -2 \frac1{\sigma^2} \mu^T \Sigma^{-1} z]} dz \\
        &\overset{*}= \frac1{\abs{\Sigma}^\frac12} e^{-\frac12 \frac1{\sigma^2} \mu^T \Sigma^{-1} \mu} \cdot \frac1{\abs{\Sigma^{-1} + \frac1{\sigma^2}I}^\frac12} e^{+\frac12 \mu^T \Sigma^{-1} (\Sigma^{-1} + \frac1{\sigma^2}I)^{-1} \Sigma^{-1} \mu} \\
        &= \frac1{\abs{\frac1{\sigma^2}\Sigma + I}^\frac12} e^{-\frac12 \frac1{\sigma^2} \mu^T [\Sigma^{-1} - \Sigma^{-1} (\Sigma^{-1} + \frac1{\sigma^2}I)^{-1} \Sigma^{-1}] \mu} \\
        &\overset{**}= \frac1{\abs{\frac1{\sigma^2}\Sigma + I}^\frac12} e^{-\frac12 \mu^T (\Sigma + \sigma^2 I)^{-1} \mu} = \frac1{\abs{\frac1{\sigma^2}\Sigma + I}^\frac12} e^{-\frac12 \frac1{\sigma^2} \mu^T (\frac1{\sigma^2}\Sigma + I)^{-1} \mu}
    \end{align*}
    $(*)$ Integral over pdf of normal variable with $(\Sigma^{-1} + \frac1{\sigma^2}I)^{-1} \Sigma^{-1} \mu$ mean and $(\Sigma^{-1} + \frac1{\sigma^2}I)$ covariance. \\
    $(**)$ Woodbury matrix identity.
\end{proof}

\begin{lemma}
\label{lemma:gauusian-cov-gradient}
The gradient of the MMD loss is given by:
\begin{multline}
    \nabla_a \MMD(P_{a^*}, P_a) = -2 \abs{\frac2{\sigma^2} (aa^T + \varepsilon^2I) + I}^{-\frac32} \(\frac1{\sigma^2}\)^d (2 \varepsilon^2 + \sigma^2)^{d-1} a \\+ 2 \abs{\frac1{\sigma^2} (aa^T + {a^*}{a^*}^T + 2\varepsilon^2I) + I}^{-\frac32} \(\frac1{\sigma^2}\)^d (1 + 2\varepsilon^2 + \sigma^2) (2\varepsilon^2 + \sigma^2)^{d-1} ({a^*}{a^*}^T + (2\varepsilon^2 + \sigma^2) I)^{-1} a
\end{multline}
\end{lemma}

\begin{proof}
    By differentiating each term of Equation \ref{eq:mmd-cov} separately:
    \begin{align*}
        \nabla_a \[\frac1{\abs{\frac2{\sigma^2} (aa^T + \varepsilon^2I) + I}^\frac12}\]
        &= -\frac12 \abs{\frac2{\sigma^2} (aa^T + \varepsilon^2I) + I}^{-\frac32} \cdot \nabla_a \abs{\frac2{\sigma^2} (aa^T + \varepsilon^2I) + I} \\
        &\overset*= -\frac12 \abs{\frac2{\sigma^2} (aa^T + \varepsilon^2I) + I}^{-\frac32} \cdot \(\frac1{\sigma^2}\)^d (2 \varepsilon^2 + \sigma^2)^d (\varepsilon^2 + \frac{\sigma^2}2)^{-1} \nabla_a [a^T a]\\
        &= -\abs{\frac2{\sigma^2} (aa^T + \varepsilon^2I) + I}^{-\frac32} \cdot \(\frac1{\sigma^2}\)^d (2 \varepsilon^2 + \sigma^2)^d (\varepsilon^2 + \frac{\sigma^2}2)^{-1} a \\
        &= -2 \abs{\frac2{\sigma^2} (aa^T + \varepsilon^2I) + I}^{-\frac32} \cdot \(\frac1{\sigma^2}\)^d (2 \varepsilon^2 + \sigma^2)^{d-1} a
    \end{align*}
    $(*)$ using Sylvester's determinant theorem:
    \begin{align*}
        \abs{\frac2{\sigma^2} (aa^T + \varepsilon^2I) + I}
        &= \(\frac2{\sigma^2}\)^d \abs{ aa^T + (\varepsilon^2 + \frac{\sigma^2}2) I} \\
        &= \(\frac2{\sigma^2}\)^d \abs{ (\varepsilon^2 + \frac{\sigma^2}2) I} \abs{1 + a^T (\varepsilon^2 + \frac{\sigma^2}2)^{-1} a} \\
        &= \(\frac1{\sigma^2}\)^d (2 \varepsilon^2 + \sigma^2)^d (1 + a^T (\varepsilon^2 + \frac{\sigma^2}2)^{-1} a)
    \end{align*}
    
    \begin{align*}
        &\nabla_a \[-2 \frac1{\abs{\frac1{\sigma^2} (aa^T + {a^*}{a^*}^T + 2\varepsilon^2I) + I}^\frac12}\] \\
        &= \abs{\frac1{\sigma^2} (aa^T + {a^*}{a^*}^T + 2\varepsilon^2I) + I}^{-\frac32} \cdot \nabla_a \abs{\frac1{\sigma^2} (aa^T + {a^*}{a^*}^T + 2\varepsilon^2I) + I} \\
        &\overset*= \abs{\frac1{\sigma^2} (aa^T + {a^*}{a^*}^T + 2\varepsilon^2I) + I}^{-\frac32} \cdot \(\frac1{\sigma^2}\)^d (1 + 2\varepsilon^2 + \sigma^2) (2\varepsilon^2 + \sigma^2)^{d-1} \nabla_a a^T ({a^*}{a^*}^T + (2\varepsilon^2 + \sigma^2) I)^{-1} a \\
        &= 2 \abs{\frac1{\sigma^2} (aa^T + {a^*}{a^*}^T + 2\varepsilon^2I) + I}^{-\frac32} \cdot \(\frac1{\sigma^2}\)^d (1 + 2\varepsilon^2 + \sigma^2) (2\varepsilon^2 + \sigma^2)^{d-1} ({a^*}{a^*}^T + (2\varepsilon^2 + \sigma^2) I)^{-1} a
    \end{align*}
    $(*)$ using Sylvester's determinant theorem (and $a^* = e_1$):
    \begin{align*}
        \abs{\frac1{\sigma^2} (aa^T + {a^*}{a^*}^T + 2\varepsilon^2I) + I}
        &= \(\frac1{\sigma^2}\)^d \abs{aa^T + ({a^*}{a^*}^T + (2\varepsilon^2 + \sigma^2) I)} \\
        &= \(\frac1{\sigma^2}\)^d \abs{{a^*}{a^*}^T + (2\varepsilon^2 + \sigma^2) I} \abs{1 + a^T ({a^*}{a^*}^T + (2\varepsilon^2 + \sigma^2) I)^{-1} a} \\
        &= \(\frac1{\sigma^2}\)^d (1 + 2\varepsilon^2 + \sigma^2) (2\varepsilon^2 + \sigma^2)^{d-1} (1 + a^T ({a^*}{a^*}^T + (2\varepsilon^2 + \sigma^2) I)^{-1} a)
    \end{align*}
\end{proof}

\begin{lemma}
\label{lemma:gausian-cov-hessian}
The Hessian of the MMD loss is given by:
    \begin{align*}
        &\nabla^2 \MMD(a) \\
        &= 2 \sigma^{-d} (2 \varepsilon^2 + \sigma^2)^{d-1} \abs{2 aa^T + (2\varepsilon^2 + \sigma^2) I}^{-\frac32} \( 6 (2 \varepsilon^2 + \sigma^2)^{d-1} \abs{2aa^T + (2\varepsilon^2 + \sigma^2) I}^{-1} a a^T - I\) \\
        &\quad + 2 \sigma^{-d} (2 \varepsilon^2 + \sigma^2)^{d-1} (1 + 2\varepsilon^2 + \sigma^2) \abs{aa^T + {a^*}{a^*}^T + (2\varepsilon^2 + \sigma^2) I}^{-\frac32} 
        \Bigg(({a^*}{a^*}^T + (2\varepsilon^2 + \sigma^2) I)^{-1}
         \\
        &\quad\quad- 3 (1 + 2\varepsilon^2 + \sigma^2) (2\varepsilon^2 + \sigma^2)^{d-1} \abs{aa^T + {a^*}{a^*}^T + (2\varepsilon^2 + \sigma^2) I}^{-1} ({a^*}{a^*}^T + (2\varepsilon^2 + \sigma^2) I)^{-1} a a^T ({a^*}{a^*}^T + (2\varepsilon^2 + \sigma^2) I)^{-1} \Bigg)
    \end{align*}
\end{lemma}

\begin{proof}
    We take the gradient of the gradient in Lemma \ref{lemma:gauusian-cov-gradient}. We begin with the first term (and set aside the constant $2 \sigma^{-2d} (2 \varepsilon^2 + \sigma^2)^{d-1}$ for simplicity).
        \begin{align*}
            &\nabla \[- \abs{\frac2{\sigma^2} (aa^T + \varepsilon^2I) + I}^{-\frac32} a\] \\
            &= - a \nabla \[ \abs{\frac2{\sigma^2} (aa^T + \varepsilon^2I) + I}^{-\frac32} \]^T - \abs{\frac2{\sigma^2} (aa^T + \varepsilon^2I) + I}^{-\frac32} \nabla \[a\] \\
            &= \frac32 \abs{\frac2{\sigma^2} (aa^T + \varepsilon^2I) + I}^{-\frac52} a \nabla \[ \abs{\frac2{\sigma^2} (aa^T + \varepsilon^2I) + I} \]^T - \abs{\frac2{\sigma^2} (aa^T + \varepsilon^2I) + I}^{-\frac32} I \\
            &\overset{*}= 6 \(\frac1{\sigma^2}\)^d (2 \varepsilon^2 + \sigma^2)^{d-1} \abs{\frac2{\sigma^2} (aa^T + \varepsilon^2I) + I}^{-\frac52} a a^T - \abs{\frac2{\sigma^2} (aa^T + \varepsilon^2I) + I}^{-\frac32} I \\
            &= \sigma^{3d} \abs{2aa^T + (2\varepsilon^2 + \sigma^2) I}^{-\frac32} \( 6 (2 \varepsilon^2 + \sigma^2)^{d-1} \abs{2aa^T + (\varepsilon^2 + \sigma^2) I}^{-1} a a^T - I\)
        \end{align*}
        $(*)$ As we found in Lemma \ref{lemma:gauusian-cov-gradient}.
        For the second term, we set aside the constant $2 \(\frac1{\sigma^2}\)^d (1 + 2\varepsilon^2 + \sigma^2) (2\varepsilon^2 + \sigma^2)^{d-1}$ for simplicity.
        \begin{align*}
            &\nabla \[ \abs{\frac1{\sigma^2} (aa^T + {a^*}{a^*}^T + 2\varepsilon^2I) + I}^{-\frac32} ({a^*}{a^*}^T + (2\varepsilon^2 + \sigma^2) I)^{-1} a \] \\
            &= ({a^*}{a^*}^T + (2\varepsilon^2 + \sigma^2) I)^{-1} a \nabla\[ \abs{\frac1{\sigma^2} (aa^T + {a^*}{a^*}^T + 2\varepsilon^2I) + I}^{-\frac32} \]^T \\&\qquad\qquad+ \abs{\frac1{\sigma^2} (aa^T + {a^*}{a^*}^T + 2\varepsilon^2I) + I}^{-\frac32} ({a^*}{a^*}^T + (2\varepsilon^2 + \sigma^2) I)^{-1} \nabla \[ a \] \\
            &= -\frac32 \abs{\frac1{\sigma^2} (aa^T + {a^*}{a^*}^T + 2\varepsilon^2I) + I}^{-\frac52} ({a^*}{a^*}^T + (2\varepsilon^2 + \sigma^2) I)^{-1} a \nabla\[ \abs{\frac1{\sigma^2} (aa^T + {a^*}{a^*}^T + 2\varepsilon^2I) + I} \]^T \\&\qquad\qquad+ \abs{\frac1{\sigma^2} (aa^T + {a^*}{a^*}^T + 2\varepsilon^2I) + I}^{-\frac32} ({a^*}{a^*}^T + (2\varepsilon^2 + \sigma^2) I)^{-1} \\
            &\overset{*}= -3 \(\frac1{\sigma^2}\)^d (1 + 2\varepsilon^2 + \sigma^2) (2\varepsilon^2 + \sigma^2)^{d-1} \abs{\frac1{\sigma^2} (aa^T + {a^*}{a^*}^T + 2\varepsilon^2I) + I}^{-\frac52} ({a^*}{a^*}^T + (2\varepsilon^2 + \sigma^2) I)^{-1} a a^T ({a^*}{a^*}^T + (2\varepsilon^2 + \sigma^2) I)^{-1} \\&\qquad\qquad+ \abs{\frac1{\sigma^2} (aa^T + {a^*}{a^*}^T + 2\varepsilon^2I) + I}^{-\frac32} ({a^*}{a^*}^T + (2\varepsilon^2 + \sigma^2) I)^{-1} \\
            &= \sigma^{3d} \abs{aa^T + {a^*}{a^*}^T + (2\varepsilon^2 + \sigma^2) I}^{-\frac32} \Bigg[ 
                ({a^*}{a^*}^T + (2\varepsilon^2 + \sigma^2) I)^{-1} - 3 (1 + 2\varepsilon^2 + \sigma^2) (2\varepsilon^2 + \sigma^2)^{d-1} \\&\qquad\qquad \cdot \abs{aa^T + {a^*}{a^*}^T + (2\varepsilon^2 + \sigma^2) I}^{-1} ({a^*}{a^*}^T + (2\varepsilon^2 + \sigma^2) I)^{-1} a a^T ({a^*}{a^*}^T + (2\varepsilon^2 + \sigma^2) I)^{-1}
            \qquad\quad\Bigg]
        \end{align*}
        $(*)$ As we found in Lemma \ref{lemma:gauusian-cov-gradient}.
\end{proof}

\begin{lemma} \label{lemma:cov_bounded}
$\norm[2]{\nabla^2 \MMD(a)}$ is bounded for all $a$.
\end{lemma}
\begin{proof}
Let $a$ be some random direction. It is enough to show that $\abs{ v^T \nabla^2 \MMD(a) v }$ is bounded above for any vector $v$ is unit size ($\norm{v}=1$). Using the Hessian (Lemma \ref{lemma:gausian-cov-hessian}) and $c:=2 \varepsilon^2 + \sigma^2$ for simplicity:
    \begin{multline}
        v^T \nabla^2 \MMD(a) v \propto \abs{2 aa^T + c I}^{-\frac32} \( 6 c^{d-1} \abs{2aa^T + c I}^{-1} (a^T v)^2 - 1\)
        + (1 + c) \abs{aa^T + {a^*}{a^*}^T + c I}^{-\frac32} 
        \\\cdot v^T \Bigg(({a^*}{a^*}^T + c I)^{-1}
         - 3 (1 + c) c^{d-1} \abs{aa^T + {a^*}{a^*}^T + c I}^{-1} ({a^*}{a^*}^T + c I)^{-1} a a^T ({a^*}{a^*}^T + c I)^{-1} \Bigg) v
    \end{multline}
By Woodbury's identity, $\(a^*{a^*}^T + c I\)^{-1} = c^{-1} I - c^{-1} (1 + c)^{-1} {a^*}{a^*}^T$, 
plugging all together and switch negative expressions by positive (triangle inequality), the following holds up to positive constant:
\begin{multline}
    \abs{v^T \nabla^2 \MMD(a) v} \le^\propto \abs{2 aa^T + c I}^{-\frac32} \( 6 c^{d-1} \abs{2aa^T + c I}^{-1} (a^T v)^2 + 1\)
    + c^{-1} (1 + c) \abs{aa^T + {a^*}{a^*}^T + c I}^{-\frac32} 
    \\\cdot \Bigg(
    1 + (1 + c)^{-1} ({a^*}^T v)^2 
     + 3 (1 + c) c^{d-2} \abs{aa^T + {a^*}{a^*}^T + c I}^{-1} \\\Big((a^T v)^2 + (1+c)^{-1} (a^T v) ({a^*}^T v) + (1+c)^{-2} (a^T a^*)^2 ({a^*}^T v)^2 \Big) \Bigg)
\end{multline}
Instead of proving for any $v$, we can take orthogonal basis such that $v_1 \perp a$, $v_2 \perp a^*$, and $v_i \perp a, a^*$ for $2 < i \le d$.

Taking $v \perp a, a^*$:
\begin{equation}
        \abs{v^T \nabla^2 \MMD(a) v} \le^\propto \abs{2 aa^T + c I}^{-\frac32} + c^{-1} (1 + c) \abs{aa^T + {a^*}{a^*}^T + c I}^{-\frac32}
        \le \abs{c I}^{-\frac32} + c^{-1} (1 + c) \abs{{a^*}{a^*}^T + c I}^{-\frac32}
\end{equation}

W.L.O.G we can assume $a$ and $a^*$ are not linearly independent (otherwise $v \perp a$ or $v \perp a^*$ are equivalent to $v \perp a, a^*$ case). Taking $v = v_1 \perp a$:
\begin{align}
    \abs{v^T \nabla^2 \MMD(a) v} &\le^\propto \abs{2 aa^T + c I}^{-\frac32}
        + c^{-1} (1 + c) \abs{aa^T + {a^*}{a^*}^T + c I}^{-\frac32} 
        \\&\qquad\qquad \cdot \Bigg(
        1 + (1 + c)^{-1} ({a^*}^T v)^2 
         + 3 (1 + c)^{-1} c^{d-2} \abs{aa^T + {a^*}{a^*}^T + c I}^{-1} (a^T a^*)^2 ({a^*}^T v)^2 \Bigg) \\
    & \le \abs{c I}^{-\frac32} + c^{-1} (1 + c) \abs{{a^*}{a^*}^T + c I}^{-\frac32} \( 1 + (1 + c)^{-1} ({a^*}^T v)^2  \) \\&\qquad\qquad\qquad\qquad\qquad\qquad\qquad\qquad\qquad\qquad + 3 c^{d-3} \abs{aa^T + c I}^{-\frac52} (a^T a^*)^2 ({a^*}^T v)^2
\end{align}

To bound the last term we use $\hat{v} = a^* - \frac{a^T a^*}{\norm{a}^2} a$ and $v = \frac{\hat{v}}{\norm{\hat{v}}}$ which maximize the last expression:
\begin{multline}
    \le \abs{c I}^{-\frac32} + c^{-1} (1 + c) \abs{{a^*}{a^*}^T + c I}^{-\frac32} \( 1 + (1 + c)^{-1} \frac{\norm{a^*}^4}{\norm{a^* - \frac{a^T a^*}{\norm{a}^2} a}^2} \(1 - \frac{(a^T a^*)^2}{\norm{a^*}^2\norm{a}^2} \)^2  \) \\+ 3 c^{d-3} \abs{aa^T + c I}^{-\frac52} (a^T a^*)^2 \frac{\norm{a^*}^4}{\norm{a^* - \frac{a^T a^*}{\norm{a}^2} a}^2} \(1 - \frac{(a^T a^*)^2}{\norm{a^*}^2 \norm{a}^2}\)^2
\end{multline}

We are left to handle the two terms that depend on $a$. Denote $\alpha = \frac{\abs{a^T a^*}^2}{\norm{a}^2\norm{a^*}^2} \in (0,1)$. Then the first term is bounded by:
\begin{equation*}
    \frac1{\norm{a^* - \frac{a^T a^*}{\norm{a}^2} a}^2} \(1 - \frac{(a^T a^*)^2}{\norm{a^*}^2\norm{a}^2} \)^2
    = \frac{\(1 - \alpha \)^2}{\norm{a^*}^2(1 - \alpha)} \le \frac1{\norm{a^*}^2}
\end{equation*}
and for the second term: 
\begin{equation}
    \abs{aa^T + c I}^{-\frac52} \frac{(a^T a^*)^2}{\norm{a^* - \frac{a^T a^*}{\norm{a}^2} a}^2} \(1 - \frac{(a^T a^*)^2}{\norm{a^*}^2 \norm{a}^2}\)^2
    \propto \( 1 + \frac{c}{\norm{a}^2} \)^{-2.5} \frac{\alpha (1 - \alpha)}{\norm{a}^3}
    \le \( 1 + \frac{c}{\norm{a}^2} \)^{-2.5} \norm{a}^{-3}
\end{equation}
Which converges to zero as $\norm{a} \to 0$ or $\norm{a} \to \infty$ (independent of the direction of $a$), so the term is bounded above as function of the size of $a$.

Finally for $v = v_2 \perp a^*$:
\begin{multline}
    \abs{v^T \nabla^2 \MMD(a) v} \le^\propto \abs{2 aa^T + c I}^{-\frac32} \( 6 c^{d-1} \abs{2aa^T + c I}^{-1} (a^T v)^2 + 1\) \\ + c^{-1} (1 + c) \abs{aa^T + {a^*}{a^*}^T + c I}^{-\frac32} \cdot \Bigg(1 + 3 (1 + c) c^{d-2} \abs{aa^T + {a^*}{a^*}^T + c I}^{-1} (a^T v)^2 \Bigg) \\
    \le 6 c^{d-1} \abs{2aa^T + c I}^{-2.5} (a^T v)^2 + \abs{c I}^{-\frac32} + c^{-1} (1 + c) \abs{{a^*}{a^*}^T + c I}^{-\frac32} \\ + 3 (1 + c)^2 c^{d-3} \abs{aa^T + c I}^{-2.5} (a^T v)^2
\end{multline}

To bound the last term we use $\hat{v} = a - \frac{a^T a^*}{\norm{a^*}^2} a^*$ and $v = \frac{\hat{v}}{\norm{\hat{v}}}$ which maximize the last expression:
\begin{multline}
    \le 6 c^{d-1} \abs{2aa^T + c I}^{-2.5} \frac{\norm{a}^4}{\norm{a - \frac{a^T a^*}{\norm{a^*}^2} a^*}^2} (1 - \alpha)^2 + \abs{c I}^{-\frac32} + c^{-1} (1 + c) \abs{{a^*}{a^*}^T + c I}^{-\frac32} \\ + 3 (1 + c)^2 c^{d-3} \abs{aa^T + c I}^{-2.5} \frac{\norm{a}^4}{\norm{a - \frac{a^T a^*}{\norm{a^*}^2} a^*}^2} (1 - \alpha)^2
\end{multline}
In this case:
\begin{equation}
    \abs{aa^T + c I}^{-2.5} \frac{\norm{a}^4 (1-\alpha)^2}{\norm{a - \frac{a^T a^*}{\norm{a^*}^2} a^*}^2} \propto (\norm{a}^2 + c)^{-2.5} \frac{\norm{a}^4 (1-\alpha)^2}{\norm{a}^2 (1 - \alpha)}
    \le \( 1 + \frac{c}{\norm{a}^2} \)^{-2.5} \norm{a}^{-3}
\end{equation}
Which we already showed to be bounded above.

We conclude that we can upper bound $\abs{v^T \nabla^2 \MMD(a) v}$ with a constant independent of $a$
\end{proof}

\begin{lemma}\label{lemma:mean_bounded}
    $\norm[2]{\nabla^2 \MMD(\mu)}$ is bounded for all $\mu$.
\end{lemma}
\begin{proof}
Given the gradient up to a multiplicative positive factor:
\begin{equation}
    \nabla \MMD(P_{\mu^*}, P_\mu) \propto 
            e^{-\frac12 c^{-1} \norm{\mu - \mu^*}^2} (\mu - \mu^*)
            \\+ e^{-\frac12 c^{-1} \norm{\mu + \mu^*}^2} (\mu + \mu^*)
            - 2 e^{-\frac12 c^{-1} \norm{2\mu}^2} \mu
\end{equation}
it is easy to compute the Hessian:
\begin{multline}
    \nabla^2 \MMD(P_{\mu^*}, P_\mu) \propto 
        e^{-\frac12 c^{-1} \norm{\mu - \mu^*}^2} \cdot \( I - c^{-1} \norm{\mu - \mu^*}^2 (\mu - \mu^*) (\mu - \mu^*)^T \)
        \\
        + e^{-\frac12 c^{-1} \norm{\mu + \mu^*}^2} \( I - c^{-1} \norm{\mu + \mu^*} (\mu + \mu^*) (\mu + \mu^*)^T \)
        - 2 e^{-\frac12 c^{-1} \norm{2\mu}^2} \( I - 2 c^{-1} \norm{2\mu} \mu \mu^T \) 
\end{multline}
For any direction $v$ in unit size (and by the triangle inequality):
\begin{multline}
    \abs{v^T \nabla^2 \MMD(P_{\mu^*}, P_\mu) v} \leq \propto
        e^{-\frac12 c^{-1} \norm{\mu - \mu^*}^2} \cdot \( 1 + c^{-1} \norm{\mu - \mu^*}^2 ((\mu - \mu^*)^T v)^2 \)
        \\
        + e^{-\frac12 c^{-1} \norm{\mu + \mu^*}^2} \( 1 + c^{-1} \norm{\mu + \mu^*} ((\mu + \mu^*)^T v)^2 \)
        + 2 e^{-\frac12 c^{-1} \norm{2\mu}^2} \( 1 + 2 c^{-1} \norm{2\mu} (\mu^T v)^2 \)
\end{multline}

For $v \perp \mu, \mu - \mu^*, \mu + \mu^*$ is easy to bound the above by $3$. For general $v$ the above will depend on the components
of $v$ in each of these three vectors. Denote one of these three vectors by $a$. Then we need to bound:
\begin{equation}
    e^{-\frac12 c^{-1} \norm{a}^2} \norm{a} (a v)^2
    \le e^{-\frac12 c^{-1} \norm{a}^2} \norm{a}^5
\end{equation}
Denote $g(t) = e^{-\frac12 c^{-1} t^2} t^5$, so:
\begin{equation}
    g'(t) =  e^{-\frac12 c^{-1} t^2} \cdot (5 - c^{-1} t^2) \cdot t^4 = 0
\end{equation}
and the above holds only for $t = 0$ or $t = \sqrt{5c}$, hence $\abs{v^T \nabla^2 \MMD(P_{\mu^*}, P_\mu) v}$ bounded by $3 + e^{-\frac52} \norm{5c}^\frac52$.
\end{proof}

\newpage

\section{EXPERIMENTS}

\subsection{Graphs Parameters}
In the following tables we describe the parameters which were used in the experiments:

\begin{table}[h!]
\caption{Experiments Details - Gaussian with Unknown Mean} \label{tbl:gaussian-mean}
\begin{center}
\begin{tabular}{lccccc}
\textbf{Loss}  &\textbf{Learning rate} &\textbf{Iterations} &\textbf{Bandwidth} &\textbf{Fakes ratio} &\textbf{Repeats} \\
\hline
MMD & $10^{-1}$ & 500 & $10^1$ & 1 & 800 \\
OSMMD & $10^{-1}$ & 500 & $10^1$ & & 100 \\
MLE & $10^{-1}$ & 500 & & & 100 \\
WGAN & $10^{-2}$ & 1000 & & 1 & 500
\end{tabular}
\end{center}
\end{table}

\begin{table}[h!]
\caption{Experiments Details - GMM with Unknown Mean} \label{tbl:gmm-mean}
\begin{center}
\begin{tabular}{lccccc}
\textbf{Loss}  &\textbf{Learning rate} &\textbf{Iterations} &\textbf{Bandwidth} &\textbf{Fakes ratio} &\textbf{Repeats} \\
\hline
MMD & $10^{-1}$ & 1000 & $10^1$ & 1 & 700 \\
OSMMD & $10^{-1}$ & 1000 & $10^1$ & & 100 \\
MLE & $10^{-1}$ & 1000 & & & 100 \\
WGAN & $10^{-1}$ & 1000 & & 1 & 1000
\end{tabular}
\end{center}
\end{table}

\begin{table}[h!]
\caption{Experiments Details - Gaussian with Unknown Covariance} \label{tbl:gaussian-cov}
\begin{center}
\begin{tabular}{lccccc}
\textbf{Loss}  &\textbf{Learning rate} &\textbf{Iterations} &\textbf{Bandwidth} &\textbf{Fakes ratio} &\textbf{Repeats} \\
\hline
MMD & $10^{-1}$ & 1000 & $10^2$ & 1 & 800 \\
OSMMD & $10^{-1}$ & 1000 & $10^1$ & & 100 \\
MLE & $10^{-1}$ & 1000 & & & 100 \\
WGAN & $10^{-1}$ & 1000 & & 1 & 800
\end{tabular}
\end{center}
\end{table}

\begin{table}[h!]
\caption{Experiments Details - GMM with Unknown Covariance} \label{tbl:gmm-cov}
\begin{center}
\begin{tabular}{lccccc}
\textbf{Loss}  &\textbf{Learning rate} &\textbf{Iterations} &\textbf{Bandwidth} &\textbf{Fakes ratio} &\textbf{Repeats} \\
\hline
MMD & $10^{-2}$ & 1000 & $10^1$ & 1 & 200 \\
OSMMD & $10^{-1}$ & 1000 & $10^1$ & & 200 \\
MLE & $10^{-1}$ & 1000 & & & 200 \\
WGAN & $10^{-2}$ & 1000 & & 1 & 200
\end{tabular}
\end{center}
\end{table}

\begin{table}[h!]
\caption{Experiments Details - Gaussian with Unknown Covariance, Epsilon Comparing} \label{tbl:gaussian-cov-epsilon}
\begin{center}
\begin{tabular}{lccccc}
\textbf{Loss}  &\textbf{Learning rate} &\textbf{Iterations} &\textbf{Bandwidth} &\textbf{Fakes ratio} &\textbf{Repeats} \\
\hline
OSMMD & $10^1$ & 5000 & $10^4$ & & 100 \\
MLE & $10^{-1}$ & 10000 & & & 100
\end{tabular}
\end{center}
\end{table}

\newpage
\subsection{Linear Unmixing}

For each noise variance we initialized in VCA, and the following setups with 256 fake sampling each iteration and over 5000 iterations:
For each noise we run the following setups:
\begin{table}[h!]
\caption{MMD and WGAN setups} \label{tbl:mmd-setups}
\begin{center}
\begin{tabular}{lcccc}
\textbf{Loss} &\textbf{Learning rate} &\textbf{Bandwidth} \\
\hline
MMD 1 & $10^{-2}$ & $10^{-2}$ \\
MMD 2 & $10^{-2}$ & $10^{-1}$ \\
MMD 3 & $10^{-2}$ & $10^{0}$ \\
MMD 4 & $10^{-1}$ & $10^{-2}$ \\
MMD 5 & $10^{-1}$ & $10^{-1}$ \\
MMD 6 & $10^{-1}$ & $10^{0}$
\end{tabular}
\quad
\begin{tabular}{lcccc}
\textbf{Loss}  &\textbf{Learning rate} &\textbf{Layers} \\
\hline
WGAN 1 & $10^{-3}$ & 2 \\
WGAN 2 & $10^{-2}$ & 2 \\
WGAN 3 & $10^{-1}$ & 2 \\
WGAN 4 & $10^{-3}$ & 3 \\
WGAN 5 & $10^{-2}$ & 3 \\
WGAN 6 & $10^{-1}$ & 3
\end{tabular}
\end{center}
\end{table}

Final results of experiments (each averaged over $50$ trials) were:

\begin{table}[h!]
\caption{Linear Unmixing results} \label{tbl:unmixing-res}
\begin{center}
\begin{tabular}{ccccccc}
\textbf{Loss} &\textbf{Random} &\textbf{VCA} \\&\textbf{MMD 1} &\textbf{MMD 2} &\textbf{MMD 3} &\textbf{MMD 4} &\textbf{MMD 5} &\textbf{MMD 6} \\ &\textbf{WGAN 1} &\textbf{WGAN 2} &\textbf{WGAN 3} &\textbf{WGAN 4} &\textbf{WGAN 5} &\textbf{WGAN 6} \\
\hline
0.01 & \makecell{7.68\\(0.72)} & \makecell{0.77\\(0.23)} \\& \makecell{0.36\\(0.19)} & \makecell{0.43\\(0.14)} & \makecell{0.55\\(0.17)} & \makecell{0.56\\(0.21)} & \makecell{0.46\\(0.17)} & \makecell{0.60\\(0.19)} \\& \makecell{0.60\\(0.23)} & \makecell{0.74\\(0.22)} & \makecell{6.38\\(15.23)} & \makecell{0.61\\(0.25)} & \makecell{0.68\\(0.20)} & \makecell{2.58\\(1.82)} \\
0.1 & \makecell{7.60\\(0.62)} & \makecell{0.97\\(0.73)} \\& \makecell{1.00\\(1.11)} & \makecell{0.56\\(0.24)} & \makecell{0.63\\(0.20)} & \makecell{5.92\\(10.25)} & \makecell{0.61\\(0.26)} & \makecell{0.67\\(0.21)} \\& \makecell{0.69\\(0.46)} & \makecell{1.36\\(0.86)} & \makecell{9.03\\(18.59)} & \makecell{0.72\\(0.61)} & \makecell{0.87\\(0.58)} & \makecell{3.64\\(2.62)} \\
1.0 & \makecell{7.46\\(0.66)} & \makecell{4.67\\(1.13)} \\& \makecell{4.68\\(1.13)} & \makecell{4.89\\(1.02)} & \makecell{4.09\\(0.81)} & \makecell{4.95\\(1.51)} & \makecell{22.95\\(1.74}) & \makecell{4.15\\(0.80)} \\& \makecell{3.77\\(0.89)} & \makecell{3.56\\(0.88)} & \makecell{13.96\\(18.04)} & \makecell{3.70\\(0.90)} & \makecell{3.42\\(0.92)} & \makecell{4.76\\(1.98)} 
\end{tabular}
\end{center}
\end{table}
 
\end{document}